\PassOptionsToPackage{colorlinks=true, citecolor=green!60!black, linkcolor=blue}{hyperref}

\documentclass[11pt]{article}

\usepackage[utf8]{inputenc} 
\usepackage[T1]{fontenc}    
\usepackage{url}            
\usepackage{hyperref}
\usepackage{booktabs}       
\usepackage{amsfonts}       
\usepackage{nicefrac}       
\usepackage{microtype}      

\usepackage{bbm,amsmath,amssymb,graphicx,url,amsthm,dsfont}
\usepackage[ruled]{algorithm2e}
\usepackage{xcolor}
\usepackage{comment}
\usepackage{fullpage}

\newtheorem{theorem}{Theorem}[section]
\newtheorem{remark}[theorem]{Remark}

\newtheorem{claim}[theorem]{Claim}

\newtheorem{example}[theorem]{Example}
\newtheorem{lemma}[theorem]{Lemma}
\newtheorem{corollary}[theorem]{Corollary}

\DeclareMathOperator*{\argmin}{argmin}
\DeclareMathOperator*{\argmax}{argmax}

\DeclareMathOperator*{\E}{\mathbb{E}}
\DeclareMathOperator*{\reals}{\mathbb{R}}

\DeclareMathOperator{\spn}{span}

\newcommand{\cU}{\mathcal{U}}

\newcommand{\cD}{\mathcal{D}}

\newcommand{\C}{\mathcal{C}}

\newcommand{\cV}{\mathcal{V}}

\newcommand{\barx}{\bar{x}}
\newcommand{\bary}{\bar{y}}
\newcommand{\hatb}{\hat{\beta}}
\newcommand{\barX}{\bar{X}}
\newcommand{\barY}{\bar{Y}}

\title{Gaming Helps! \\ Learning from Strategic Interactions in Natural Dynamics}

    \author{
    Yahav Bechavod\thanks{School of Computer Science and Engineering, The Hebrew University. Email: \texttt{yahav.bechavod@cs.huji.ac.il}.}
    \and
    Katrina Ligett\thanks{School of Computer Science and Engineering, The Hebrew University. Email: \texttt{katrina@cs.huji.ac.il}.}
    \and
    Zhiwei Steven Wu\thanks{School of Computer Science, Carnegie Mellon University. Email: \texttt{zstevenwu@cmu.edu}.}
    \and
    Juba Ziani\thanks{Warren Center for Network and Data Sciences, University of Pennsylvania. Email: \texttt{jziani@seas.upenn.edu}.}
    }

\begin{document}

\maketitle

\begin{abstract}
We consider an online regression setting in which individuals adapt to the regression model: arriving individuals are aware of the current model, and invest strategically in modifying their own features so as to improve the predicted score that the current model assigns to them. Such feature manipulation has been observed in various scenarios---from credit assessment to school admissions---posing a challenge for the learner.  Surprisingly, we find that such strategic manipulations may in fact help the learner recover the meaningful variables---that is, the features that, when changed, affect the true label (as opposed to non-meaningful features that have no effect). We show that even simple behavior on the learner's part allows her to simultaneously i) accurately recover the meaningful features, and ii) incentivize agents to invest in these meaningful features, providing incentives for improvement.
\end{abstract}

\newpage

\section{Introduction}

As algorithmic decision-making takes a more and more important role in myriad application domains, incentives emerge to change the inputs presented to these algorithms. 
Recently, a collection of very interesting papers has explored various models of strategic behavior on the part of the classified individuals in learning settings, and ways to mitigate the harms to accuracy that can arise from falsified features~\cite{dalvi2004,bruckner2012,Hardt15,Dong17}. Additionally, some recent work has focused on the design of learning algorithms that incentivize the classified individuals to make ``good'' investments in true changes to their variables~\cite{Kleinberg18}.

The present paper takes a different tack, and explores another potential effect of strategic investment in true changes to variables, in an online learning setting: we claim that interaction between the online learning and the strategic individuals may actually aid the learning algorithm in identifying \emph{meaningful} variables. By meaningful, we mean, informally, and within the context of this paper, variables for which changing their true value affects the true label and thus, may lead agents to improve. In contrast, \emph{non-meaningful} variables do not affect the true label; such features are susceptible to gaming, as they can potentially be used to obtain better outcomes with respect to the posted model without actually improving true labels.

The idea is quite simple. First, if a learning algorithm's  hypothesis at a particular round depends heavily on a certain variable, this incentivizes the arriving individual to invest in improving that variable. If that variable were meaningful (that is, it has an effect on the true label), then the learner would observe an 
improved true label, increasing the observed correlation between the variable and the label. However, if that variable were non-meaningful, the changes would not have an effect on the true label, reducing the observed correlation between the variable and the label. Second, if a learning algorithm improves its hypotheses over time, this changing sequence of incentives should encourage investment in a variety of promising variables, exposing those that are meaningful. This process should naturally induce the learner to shift its dependence towards meaningful variables, thereby incentivizing individuals to invest in improving as opposed to gaming, resulting in an overall higher-quality population.

The goal of this paper is to highlight this potential beneficial effect of the interaction between online learning and strategic modification. To do so, we choose to focus our study on a simple linear regression setting. In our model, there is a true underlying latent regression parameter vector $\beta^*$, and there is an underlying distribution over unmodified feature vectors. On every round $t$, the learner must announce a regression vector $\hat{\beta}_t$.\footnote{Eventually, the learner we will consider does not update its regression vector at every round, but rather periodically, so that individuals can be treated in batches.} An individual then appears, with an unmodified feature vector $x_t$ chosen i.i.d.~from the distribution. Before presenting himself to the learner, the individual observes $\hat{\beta}_t$ and has the opportunity to invest in changing his true features to some $\bar{x}_t$; we focus on a simple model wherein the individual's investment results in a targeted change to a single variable. The individual then receives utility $\langle \hatb^t , \bar{x}_t \rangle$, and the learner gets feedback $\bary_t = \beta^*{^\intercal} \barx_t + \varepsilon_t$, where $\varepsilon_t$ is some noise. 

Within this simple model, we consider simple behaviors for both the learner and the individuals: At each time $t$, the individual modifies his features so as to maximize his utility given the posted $\hat{\beta}_t$; periodically, the learner updates $\hat{\beta}_t$ with her best estimate of $\beta^*$ given the (modified) features and labels she has observed, via least-square regression. Our main result is that under this simple behavior, the learner recovers $\beta^*$ accurately, after observing sufficiently many individuals. Our result is divided in two parts: first, we show that least-square regression accurately recovers $\beta^*$ with respect to features that many individuals have invested in. Second, we show that these dynamics incentivize investments in every feature, leading to accurate recovery of $\beta^*$ in its entirety, under an assumption on how the learner breaks ties between multiple least-square solutions. Our accuracy guarantees for a feature improve with the number of times that feature is invested in.

It is important to emphasize that we focus on a setting in which individuals' modifications (which we refer to interchangeably as ``manipulations'') of their variables can be true investments (e.g., studying to achieve better mastery of material before an exam---the exam score is the variable and the mastery level is the label) rather than deceitful manipulations (e.g., cheating on the exam to achieve a higher score without improving mastery). Deceitful manipulations would not help to expose meaningful variables, because such changes would never affect the true label (subject mastery), regardless of whether the manipulations were in meaningful or non-meaningful variables.

Notice that any discovery of meaningful variables that occurs in our model is a result of the \emph{interaction} between the online learner and the strategic individuals. On the one hand, online learning with no strategic response has no ability to distinguish non-meaningful variables from meaningful ones when the two are correlated. On the other hand, if strategic individuals faced with a static scoring algorithm tried to maximize their scores by investing in a non-meaningful feature, the resulting information would be insufficient for an observer to draw conclusions about whether other features are meaningful or not.

For example, historical data might show that both a student's grades in high school and the make of car his parents drive to the university visit day are predictive of success in university. Suppose, for simplicity, that success in high school is causally related to success in university, but that make of parents' car is not, and is merely a \emph{proxy} for other features that control one's chances of success in college. 

If the university admissions process put large weight on high school grades, that would incentivize students to invest effort in performing well in high school, which would also observably pay off in university, which would reinforce the emphasis on high school grades. If the admissions process put large weight on the make of car in which students arrive to the visit day, that would incentivize renting fancy cars for visits. However, this would result in a different distribution over the observed student variables, and on this modified distribution the correlation between cars and university success would be \emph{weakened}, and therefore the admissions formula would not perform well. In future years, the university would naturally correct the formula to de-emphasize cars.

It is important to note that our work operates under a simplifying assumption with regards to the underlying structure of the problem (introduced in Section \ref{sec:model}). Adding an assumption of this kind is necessary since in the general case recovering the exact model structure is hard. 
Our work thus aims to bring attention to a natural mechanism, based on re-training, for exposing meaningful variables, that we believe is worthy of further attention.
\section{Related Work}

Much of the work on learning assumes that an individual's data is a fixed input that is independent of the algorithm used by the decision-maker. In practice, however, individuals may try to adapt to the model in place in order to improve their outcomes. A recent line of work studies such strategic behavior in classification settings.  Part of this line of work concerns itself with the negative consequences of strategic behavior, when individuals aim to game the model in place; for example, individuals may manipulate their data or features (often at a cost) in an effort to obtain positive qualification outcomes or otherwise manipulate an algorithm's output~\cite{dalvi2004,PP04,DFP10,bruckner2012,IL13,HIM14,CDP15,Hardt15,Dong17,chara1,chara2} or even to protect their privacy~\cite{GLRS14,CIL15}. The goal in these results is to provide algorithms whose outputs are robust to such gaming. \cite{milli2019} and~\cite{lily2019} focus on the social impact of robust classification, and show that i) robust classifiers come at a social cost (by forcing even qualified individuals to invest in costly feature manipulations in order to be classified positively) and ii) disparate abilities to game the model inevitably lead to unfair outcomes. 

Another part of this line of work instead sees strategic manipulation as possibly positive, when the classifier incentivizes individuals to invest in true improvements to their features---e.g., a student may decide to study and actually improve his knowledge in order to raise his test score.~\cite{Kleinberg18},~\cite{berk2019},~\cite{planck1},~\cite{planck2} and~\cite{nika} study how to incentivize agents to invest effort in modifying meaningful features that improve their labels. 
Much of this line of work assumes that the decision-maker already understands which features are meaningful and affect agents' labels or outcomes, and which do not. 

In contrast, we consider a setting where the decision-maker does not initially know which features affect agents' labels, and aims to leverage the agents' strategic behavior to expose what features these are. Most closely related to this paper is the work of~\cite{miller2019}, as well as the concurrent works of~\cite{moritz_performative} and~\cite{yo}.~\cite{miller2019} formalize the distinction between gaming and actual improvements by drawing a connection to causality and introducing causal graphs that model the effects of the features and target variables on each other. They show that in such settings, the decision-maker should incentivize actual improvements rather than gaming, and that designing good incentives that push agents to improve is at least as hard as causal inference. \cite{yo} study the sample complexity of learning a linear regression model so as to either i) maximize the accuracy of the predictions, ii) maximize the agents' self-improvements, or iii) recover the causality structure of their problem.~\cite{moritz_performative} show how re-training can lead to stable and optimal outcomes when the learner's model affects the distribution of agent features and labels; while our paper considers a similar re-training framework, our assumptions differ from those of~\cite{moritz_performative}.
\section{Model} \label{sec:model}

We consider a linear regression setting where the learner estimates the regression parameters based on strategically manipulated data from a sequence of agents over rounds. There is a true latent regression parameter $\beta^* \in [-1,1]^d$ that \emph{generates} an agent's label as a function of his feature vector. 
That is, for any agent with feature vector $x \in [-1,1]^d$, the real-valued label $y$ is obtained via $y = \beta^{*\top} x + \varepsilon$, where $\varepsilon$ is a noise random variable with $|\varepsilon| \leq \sigma$, and $\E [\varepsilon \mid x] = 0$. We also refer to an individual's features as variables. There is a distribution over the \emph{unmodified} features $x$ in $[-1 , 1]^d$; we let $\mu$ be the mean and $\Sigma$ be the covariance matrix of this distribution; we note that the distribution of unmodified features may be degenerate, i.e., $\Sigma$ may not be full-rank. For example, this can happen in settings in which the non-meaningful features are merely proxies for the meaningful features (i.e., those that really control the label); in that case, one may imagine that the non-meaningful features are (possibly randomized) functions of the meaningful features, leading in particular to low-rank observations when few features are meaningful.

Throughout the paper, we set $\mu = 0$.\footnote{This can be done whenever the learner can estimate the mean feature vector, since the learner can then center the features. The learner could estimate the mean by using unlabeled historical data; for example, she could collect data during a period when the algorithm does not make any decision on the agents, thus they would have no incentive to modify their features.}

The agents and the learner interact in an online fashion. At time $t$, the learner first posts a regression estimate $\hatb^t\in \reals^d$, then an agent (indexed by $t$) arrives with their unmodified feature vector $x_t$. Agent $t$ modifies the feature $x_t$ into $\barx_t$ in response to $\hatb^t$, in order to improve their assigned score $\langle \hatb^t , \bar{x}_t \rangle$. Finally, the learner observes the agent's realized label after feature modification, given by $\bary_t = \beta^*{^\intercal} \barx_t + \varepsilon_t$. 

\paragraph{Meaningful vs non-meaningful features.} 
When an agent modifies a feature $k$, this may also affect the agent's true label. We divide the coordinates of any feature vector $x$ into \emph{meaningful} and \emph{non-meaningful} features; meaningful features inform and control an agent's label, while non-meaningful features are those that can be manipulated without directly affecting an agent's label. (One can think, intuitively, of the meaningful features as causal, and the non-meaningful features as non-causal, but the language of causality is typically reserved for more complex settings than ours.) Formally, for any $k \in [d]$, feature $k$ is meaningful if and only if the coordinate $\beta^*(k) \neq 0$, and non-meaningful if and only if $\beta^*(k) = 0$. An agent $t$ can modify his true label by modifying meaningful features. As such, note that $\beta^*$ captures the underlying model structure of our problem. The magnitude of each feature in $\beta^*$ captures the extent to which said feature is meaningful and affects the agents' labels.

We remark that strategic agents---that best-respond to the learner's model to improve their regression outcomes---may at times have incentives to manipulate a feature $k$ such that $\beta^*(k) = 0$; this can happen when the learner sets $\hat{\beta}(k) \neq 0$. In such cases, agents can improve their regression outcomes \emph{without} improving their true label, which we refer to as \emph{gaming}. When agents modify a feature $k$ that aligns with the true model, we refer to such a modification as an \emph{improvement}.

\paragraph{Agents' responses.} Agents are \emph{strategic}: they modify their features so as to maximize their own regression outcome;\footnote{Importantly, our agents' goal \emph{is not} to cooperate with the learner. Agents are self-interested and aim to maximize their own regression outcomes; they do not actively seek to help the learner improve the accuracy of her model. The agents prefer when the learner emphasizes features that are easier to manipulate, even if said features are non-meaningful. These incentives may be ill-aligned with the learner's goal of optimizing predictive power and recovering model structure, which requires putting more weight on meaningful features.} modifications are costly and agents are budgeted. We assume agent $t$ incurs a linear cost $c_t(\Delta_t) = \sum_{k = 1}^d c_t(k) \left\vert \Delta_t(k) \right\vert$ to change his features by $\Delta_t$, and has a total budget of $B_t$ to modify his features. $\left(\{c_t(k)\}_{k \in [d]},B_t\right)$'s are drawn i.i.d. from a distribution $\mathcal{C}$ that is unknown to the learner. We assume $\mathcal{C}$ has discrete support $\{\left(c^1,B^1\right),\ldots, \left(c^l,B^l\right)\}$, and we denote by $\pi^i$ the probability that $\left(c_t,B_t\right) = \left(c^i,B^i\right)$. We assume $c^i(k) > 0,~B^i > 0$ for all $i \in [l],~k\in[d]$; that is, every agent can modify his features, but no feature can be modified for free.\footnote{In our model, modifying a feature affects only that feature and the label, but does not affect the values of any other features. We leave exploration of more complex models of feature intervention to future work.} When facing regression parameters $\hatb$, agent $t$ solves 
\begin{align*}
M(\hatb,c_t,B_t) = &\argmax_{\Delta_t}~~\hatb^\top \left(x_t + \Delta_t \right)\\ &\text{s.t.} \sum_{k=1}^d c_t(k)  \left\vert \Delta_t(k) \right\vert \leq B_t;
\end{align*}
That is, agent $t$ strategically aims to maximize his predicted outcome given a budget of $B$ for modifying his features, when facing model $\hatb$. The solution of the above program does not depend on $x_t$, only on $\hatb$ and $(c_t,B_t)$, and is given by
\[
\Delta_t = \sum_{k = 1}^d sgn\left(\hatb(k)\right)  \mathds{1} \left\{k = \argmax_j \left\vert\hatb(j)\right\vert/c_t(j)\right\} \frac{B_t}{c_t(k)},
\]
up to tie-breaking; when several features maximize $\vert\hatb(j)\vert/c_t(j)$, the agent modifies a single one of these features. We call $D_\tau$ the set of features that have been modified by at least one agent $t \in [\tau]$.

\begin{remark}
We make the linearity assumption on the cost functions for simplicity. Our results extend to a more general class of cost functions that do not induce modifications wherein several features are modified in a perfectly correlated fashion. 

The key technical insight we need is that the manipulations are full-rank in the subspace defined by the features that have been manipulated so far, defined as $\mathcal{V}_{\tau(E)}$ in the paper. Very strong feature correlations (which may also be thought of possible ``directions'' for modification) imply a very small minimum eigenvalue of the observation matrix, making recovery harder and increasing sample complexity. This is unavoidable: the more features are correlated, the harder they are to distinguish information-theoretically; if two features were perfectly correlated, it would be impossible to know which one affected the label.

In Theorem~\ref{thm: accurate_recovery}, we encode this correlation between modification across features in a parameter we call $\lambda$. As feature modifications become more and more correlated, the value of $\lambda$ becomes smaller and our recovery guarantees weaken. 
\end{remark}

\paragraph{Natural learner dynamics: batch least-squares regression.}

Our goal here is to identify simple, natural learning dynamics that expose  meaningful variables. Note that a simple way for the learner to expose and leverage meaningful variables is to use an explore-first then exploit type of algorithm: initially, the learner can post a model that focuses on a single feature at a time to observe how changing this feature affects the distribution of agents labels. After sequentially exploring each feature, the learner obtains an accurate estimate of $\beta^*$ that she can deploy for the remainder of the time horizon. However, one may want to avoid such an approach that artificially separates features in practice: posting models that ignore most of an agent's attribute for the sake of learning may not be desirable in real life. A bank may not want to offer loans “blindly” and willingly ignore most of a customer’s data when making lending decisions just for the purpose of learning which features are predictive of an agent's ability to repay loans. Instead, in this paper, we focus on algorithms based on \emph{re-training}: i.e., periodically, the learner updates her model based on the data she has observed so far, so as to keep it consistent with the history of agent behavior. A bank may be willing to periodically update their loan decision rule in order to keep up with new, unexpected agent behavior. 
While re-training leads to more natural dynamics than a ``naive'' explore-then-exploit approach, it comes with new technical challenges. In particular, periodic re-training leads to \emph{adaptivity}: indeed, as the model posted in the current period depend on past data, and the agents' strategic behavior depends on the model in place, the observed \emph{modified} data in each period depends on the data in all previous periods. In turn, we cannot treat data points as independent across periods.

The dynamics we consider are formally given in Algorithm~\ref{alg: lse_dynamics}. It is possible that more sophisticated learning algorithms could yield better guarantees with respect to regret and recovery; the focus of this paper is on simple and natural dynamics rather than optimal ones.

When the learner updates her regression parameters, say at time $\tau$, she does so based on the agent data observed up until time $\tau$. We model the learner as picking $\hatb$ from the set $LSE(\tau)$ of solutions to the least-square regression problem run on the agents' data up until time $\tau$, formally defined as
\[
LSE(\tau) = \argmin_{\beta} \sum_{t = 1}^{\tau} \left(\barx_t^\top \beta - \bary_t\right)^2.
\]
We introduce notation that will be useful for regression analysis. We let $\barX_\tau \in \reals_{\tau \times d}$ be the matrix of (modified) observations up until time $\tau$. Each row corresponds to an agent $t \in [\tau]$, and agent $t$'s row is given by $\barx_t^\top$. Similarly, let $\barY_\tau = \left(\bary_t\right)_{t \in [\tau]}^\top \in \reals^{\tau \times 1}$. We can rewrite, for any $\tau$,
\begin{align}\label{eq:least-square-step}
LSE(\tau) = \argmin_{\beta} \left(\barX_\tau \beta - \barY_\tau \right)^\top \left(\barX_\tau \beta - \barY_\tau \right).
\end{align}

\paragraph{Agents are grouped in epochs.} The time horizon $T$ is divided into epochs of size $n$, where $n$ is chosen by the learner. At the start of every epoch $E$, the learner updates the posted regression parameter vector as a function of the history of $\barx_t,\bary_t$ up until epoch $E$. We let $\tau(E) = En$ denote the last time step of epoch $E$.  $D_{\tau(E)}$ denotes the set of features that have been modified by at least one agent by the end of epoch $E$.

\begin{algorithm}[ht!]
    Learner picks (any) initial $\hatb_{0}$.\\
    \For{every epoch $E \in \mathbb{N}$}{
        \For {$t \in \{(E-1) n + 1, \ldots, En \}$}{
        
                Agent $t$ reports $\bar{x}_t \in M(\hatb_{E-1},c_t,B_t)$.
                
                Learner observes $\bar{y}_t = \beta^{*\top} \bar{x}_t + \varepsilon_t$.
        }
        Learner picks $\hatb_{E} \in LSE\left(\tau(E)\right)$.
    }
\caption{Online Regression with Epoch-Based Strategic modification (Epoch size $n$)}\label{alg: lse_dynamics}
\end{algorithm}

\paragraph{Examples} 
We first illustrate why unmodified observations are insufficient for \emph{any} algorithm to distinguish meaningful from non-meaningful features. Consider a setting where non-meaningful features, as merely proxies for the meaningful features, 
are in fact convex combinations of these meaningful features in the underlying (unmodified) distribution. Absent additional information, a learner would be faced with degenerate sets of observations that have rank strictly less than $d$, which can make accurate recovery of the model structure impossible:
\begin{example}\label{ex: degenerate} 
Suppose $d = 2$, $\beta^* = (1,0)$. Suppose feature $1$ is meaningful and feature $2$ is non-meaningful and is correlated with $1$: the distribution of unmodified features is such that for any feature vector $x$, feature $2$ is identical to feature $1$ as $x(2) = x(1)$. Then, any regression parameter of the form $\beta(\alpha) = (\alpha,1-\alpha)$ for $\alpha \in \reals$ assigns agents the same score as $\beta^*$. Indeed,
\[
\beta^{*\top} x = x(1) = \alpha x(1) + (1-\alpha) x(2) = \beta(\alpha)^\top x.
\]
In turn, in the absence of additional information other than the observed features and labels, $\beta^*$ is indistinguishable from any $\beta(\alpha)$, many of which recover the model structure poorly (e.g., consider any $\alpha$ bounded away from $1$).
\end{example}

At this point, a reader may wonder why it is important in Example~\ref{ex: degenerate} to recover the true model $\beta^*$, rather than simply \emph{any} vector $\beta$ that is consistent with all the data observed so far. A major reason to do so is because only the true model $\beta^*$ can guarantee robustness in response to agent modifications, and accurately predict labels \emph{after} agents have changed their features. This is illustrated in Example~\ref{ex: bad_prediction} below:
\begin{example}\label{ex: bad_prediction}
Consider the setting of Example~\ref{ex: degenerate}, and imagine agents have much lower cost for manipulating feature $2$ than feature $1$. Then, posting a regression parameter vector of the form $(\alpha,1-\alpha)$ where $\alpha$ is small enough may lead agents to modify the second, non-meaningful feature. When facing such a modification of the form $\Delta = (0, \Delta(2))$, $(\alpha,1-\alpha)$ predicts label
\[
\alpha x(1) + (1-\alpha) \left(x(2) + \Delta(2)\right) = x(1) + (1-\alpha) \Delta(2),
\]
for an agent with $x(1) = x(2)$, while the true label is given by $\beta^{*\top } (x + \Delta) = x(1)$. In turn, the predicted and true labels are different for any $\alpha \neq 1$.
\end{example}

We next illustrate that strategic agent modifications may aid in recovery of meaningful features, but only for those features that individuals actually invest in changing:
\begin{example}\label{ex: manipulation_helps}
Consider a setting where $d=3$, feature $1$ is meaningful, and features $2$ and $3$ are non-meaningful and are correlated with feature $1$ as follows: for any feature vector $x$, $x(2),x(3) = x(1)$. Let $\beta^* = (1,0,0)$. Consider a situation in which the labels are noiseless (i.e., $\varepsilon = 0$ almost surely). Suppose that agents only modify their meaningful feature by a (possibly random) amount $\Delta(1)$. 

Note that the difference (in absolute value) between the score obtained by applying a given regression parameter $\hatb$ and the score obtained by applying $\beta^*$ to feature vector $x$ is given by
\begin{align*}
\left\vert \hatb^\top x - \beta^{*\top} x \right\vert &=\big\vert 
\hatb(1) \left(x(1) + \Delta(1)\right) + \hatb(2) x(2) + \hatb(3) x(3) - x(1) - \Delta(1)
\big\vert
\\&= \left\vert 
\left(\hatb(1) + \hatb(2) + \hatb(3) - 1 \right) x(1) + \left(\hatb(1) - 1 \right) \Delta(1) 
\right\vert.
\end{align*}
In particular, for appropriate distributions of $x$ and $\Delta(1)$, the predictions of $\hatb$ and $\beta^*$ coincide if only if $\hatb(1) = 1$ and $\hatb(2) = - \hatb(3)$. As such, the learner learns after enough observations that necessarily, $\beta^*(1) = 1$. However, any regression parameter vector with $\hatb(1) = 1$, $\hatb(2) + \hatb(3) = 0$ is indistinguishable from $\beta^*$, and accurate recovery of $\beta^*(2)$ and $\beta^*(3)$ is impossible.
\end{example}
Note that even in the noiseless setting of Example~\ref{ex: manipulation_helps}, only the feature that has been modified can be recovered accurately. In more complex settings where the true labels are noisy, one should not hope to recover every feature well, but rather only those that have been modified sufficiently many times.
\section{Recovery Guarantees for Modified Features}

In this section, we focus on characterizing the recovery guarantees (with respect to the $\ell_2$-norm) of Algorithm~\ref{alg: lse_dynamics} at time $\tau(E) = E n$ for \emph{any} epoch $E$, with respect to the features that have been modified up until $\tau(E)$ (that is, in epochs $1$ to $E$). We leave discussion of how the dynamics shape the set $D_{\tau(E)}$ of modified features to Section~\ref{sec:tiebreaking}. 

The main result of this section guarantees the accuracy of the $\hatb_E$ that the learning process converges to in its interaction with a sequence of strategic agents. The accuracy of the $\hatb_E$ that is recovered for a particular feature naturally depends on the number of epochs in which that feature is modified by the agents. For a feature that is never modified, we have no ability to distinguish whether it is meaningful or not. Recovery improves as the number of observations of the modified variable increases.

Formally, our recovery guarantee is given by the following theorem:
\begin{theorem}[$\ell_2$ Recovery Guarantee for Modified Features]\label{thm: accurate_recovery}
Pick any epoch $E$. With probability at least $1-\delta$, for $n \geq \frac{\kappa d^2}{\lambda} \sqrt{\tau(E) \log(12d/\delta)}$,
\begin{align*}
\sqrt{ \sum_{k \in D_{\tau(E)}} \left(\hatb_E(k) - \beta^*(k)\right)^2}
\leq \frac{K \sqrt{d\tau(E) \log(4d/\delta)}}{\lambda n},
\end{align*}
where $K,~\kappa,~\lambda$ are instance-specific constants that only depend on $\sigma$, $\C$, $\Sigma$, such that $\lambda > 0$. 
\end{theorem}
When the epoch size is chosen so that $n = \Omega\left(\tau(E)^\alpha\right)$ for $\alpha > 1/2$, our recovery guarantee improves as $\tau(E)$ becomes larger. Now, let us fix $\tau(E) = T$ as the time horizon, and study how the relationship between $E$ and $n$ at fixed $\tau(E)$ affects the recovery guarantees. 
When $n = \Theta(\tau(E))$ (equivalently, $E = \Theta(1)$, and agents are grouped in a small, constant number of epochs),
our bound becomes $O(1/\sqrt{\tau(E)})$; this matches the well-known recovery guarantees of least square regression for a single batch of $\tau(E)$ i.i.d observations drawn from a non-degenerate distribution of features. When the epoch size $n$ is sub-linear in $\tau(E)$ (i.e., $E \gg 1$, and agents are grouped in more numerous but smaller epochs),
the accuracy guarantee degrades to $O (\sqrt{\tau(E)}/n)$, where $ \sqrt{\tau(E)}/n \gg \frac{1}{\sqrt{\tau(E)}}$. This is because some features may be modified only in a small number of epochs,\footnote{In particular, as we will see, we expect correlated, non-meaningful features to only be modified in a small number of epochs: once a non-meaningful feature $k$ has been modified in a few epochs, it is accurately recovered. In further periods $E$, the learner sets $\hatb_E(k)$ close to $0$. This disincentivizes further modifications of feature $k$.} that is, $\Theta(n)$ times, and the number of times such features are modified drives how accurately they can be recovered. 
\begin{proof}[Proof sketch for Theorem~\ref{thm: accurate_recovery}]
Full proof in Appendix~\ref{app: accurate_recovery}. We focus on the subspace $\cV_{\tau(E)}$ of $\reals^d$ spanned by the observed features $\barx_1,\ldots,\barx_{\tau(E)}$, and for any $z \in \reals^d$, we denote by $z(\cV_{\tau(E)})$ the projection of of $z$ onto $\cV_{\tau(E)}$. First, we show via concentration that in this subspace, the mean-square error is strongly convex, with parameter $\Theta(n)$ (see Claim~\ref{cor: lower_bound_norm}). This strong convexity parameter is controlled by the smallest eigenvalue of $\bar{X}_{\tau(E)}^\top \bar{X}_{\tau(E)}$ over subspace $\cV_{\tau(E)}$. Formally, we lower bound this eigenvalue and show that with probability at least $1- \delta/2$, for $n$ large enough,
\begin{align}\label{eq: sketch_lb}
\left( \hatb_E(\cV_{\tau(E)}) - \beta^*(\cV_{\tau(E)}) \right)^\top \bar{X}_{\tau(E)}^\top \bar{X}_{\tau(E)}\left( \hatb_E(\cV_{\tau(E)}) - \beta^*(\cV_{\tau(E)}) \right)
\geq \frac{\lambda n}{4}.
\end{align}
Second, we bound the effect of the noise $\varepsilon$ on the mean-squared error by $O(\sqrt{\tau(E)})$ in Lemma \ref{lem: upper_bound_norm}, once again via concentration. Formally, we abuse notation and let $\varepsilon_{\tau(E)} \triangleq \left(\varepsilon_t\right)_{t \in [\tau(E)]}^\top$, and show that with probability at least $1-\delta/2$,
\begin{align}\label{eq: sketch_ub}
\left( \hatb_E(\cV_{\tau(E)}) - \beta^*(\cV_{\tau(E)}) \right)^\top \barX_{\tau(E)}^\top \varepsilon_{\tau(E)}
\leq \left\Vert  \hatb_E(\cV_{\tau(E)}) - \beta^*(\cV_{\tau(E)}) \right\Vert_2 \cdot K \sqrt{d\tau(E) \log(4d/\delta)}.
\end{align}
Finally, we obtain the result via Lemma~\ref{lem: FOC}, 
that states that taking the first-order conditions on the mean-squared error yields
$$
\bar{X}_{\tau(E)}^\top \bar{X}_{\tau(E)} \left( \hatb_E(\cV_{\tau(E)}) - \beta^*(\cV_{\tau(E)}) \right)=\barX_{\tau(E)}^\top \varepsilon_{\tau(E)},
$$
which can be combined with Equations~\eqref{eq: sketch_lb} and~\eqref{eq: sketch_ub} to show our bound with respect to sub-space $\cV_{\tau(E)}$. In turn, 
as $D_{\tau(E)}$ defines a sub-space of $\cV_{\tau(E)}$, our accuracy bound applies to $D_{\tau(E)}$.
\end{proof}

\begin{remark}
Theorem~\ref{thm: accurate_recovery} is not a direct consequence of the classical recovery guarantees of least-square regression, as they assume $\bar{X}_{\tau(E)}^\top \bar{X}_{\tau(E)}$ has full rank $d$. We deal with degenerate distributions over modified features, that can arise in our setting as per Examples~\ref{ex: degenerate} and~\ref{ex: manipulation_helps}.
\end{remark}
\section{Exploration via Least Squares Tie-Breaking}\label{sec:tiebreaking}

In this section, we show that a natural tie-breaking rule among the set of least squares incentivizes agents' modification of a diverse set of variables over time.

Recall we are solving the least-square problem $LSE(\tau(E))$ given in Equation~\eqref{eq:least-square-step} for all epochs $E$. 
When $\barX_{\tau(E)}^\top\barX_{\tau(E)}$ is invertible, this has a single solution. 
However, in our setting, it may be the case that $\barX_{\tau(E)}^\top\barX_{\tau(E)}$ is rank-deficient (see Examples~\ref{ex: degenerate},~\ref{ex: manipulation_helps}). In this case, the 
least-square problem admits a continuum of solutions. This gives rise to the question of which solutions are preferable in our setting, and how to break ties between several solutions. 

The learner's choice of regression parameters in each epoch affects the distribution of feature modifications in subsequent epochs. As the recovery guarantee of Theorem~\ref{thm: accurate_recovery} only applies to features that have been modified, we would like our tie-breaking rule to regularly incentivize agents to modify new features. We first show that a natural, commonly used tie-breaking rule---picking the minimum norm solution to the least-square problem---may fail to do so:

\begin{example}\label{ex: tie-breaking}
Consider a setting with $d=2$, $\beta^* = (1,2)$ and noiseless labels, i.e., $\varepsilon_t = 0$ always. Suppose that with probability $1$, every agent $t$ has features $x_t = (0,0)$, budget $B_t = 1$, and costs $c_t(1) = c_t(2) = 1$ to modify each feature. We let the tie-breaking pick the solution with the least $\ell_2$ norm among all solutions to the least-square problem. 

Pick any initial regression parameter $\hatb_0$ with $\hatb_0(1) > \hatb_0(2)$. For every agent $t$ in epoch $1$, $t$ picks modification vector $\Delta_t = (1,0)$. This induces observations $\barx_t = (1,0)$, $\bary_t = 1$. The set of least-square solutions (with error exactly $0$) in epoch $1$ is then given by $\{(1,\beta_2):~\forall \beta_2 \in \reals\}$, and the minimum-norm solution chosen at the end of epoch $1$ is $\hatb_1 = (1,0)$. This solution incentivizes agents to set $\Delta_t = (1,0)$, and Algorithm~\ref{alg: lse_dynamics} gets stuck in a loop where every agent $t$ reports $\barx_t = (1,0)$, and the algorithm posts regression parameter vector $\hatb_E = (1,0)$ in response, in every epoch $E$. The second feature is never modified by any agent, and is not recovered accurately. 
\end{example}

Example~\ref{ex: tie-breaking} highlights that a wrong choice of tie-breaking rule can lead Algorithm~\ref{alg: lse_dynamics} to explore the same features over and over again. In response, we propose the following tie-breaking rule, described in Algorithm~\ref{alg: tie-breaking}: 
\begin{algorithm}[ht]
\textbf{Input: }{Epoch $E$, observations $(\barx_1,\bary_1),\dots,(\barx_{\tau(E)},\bary_{\tau(E)})$, parameter $\alpha$}

Let $\cU_{\tau(E)} = \spn \left(\barx_1,\ldots,\barx_{\tau(E)}\right)$.

\uIf{
$rank\left(\cU_{\tau(E)}\right) < d $
}
{
Find an orthonormal basis $B_{\tau(E)}^\bot$ for $\cU^\bot_{\tau(E)}$.

Set $v=\sum_{b \in B^\bot_{\tau(E)}} b \neq 0$, renormalize $v := \frac{v}{\left\Vert v \right\Vert_2}$.

Pick $\beta_E$ a vector in $LSE(\tau(E))$ with minimal norm.

Set $\hat{\beta}_E = \beta_E + \alpha v$.
}
\Else{
Set $\hat{\beta}_E$ be the unique element in $LSE(\tau(E))$.
}
\textbf{Output: } $\hat{\beta}_E$.
\caption{Tie-Breaking Scheme at Time $\tau(E)$.}\label{alg: tie-breaking}
\end{algorithm}
Intuitively, at the end of epoch $E$, our tie-breaking rule picks a solution in $LSE(\tau(E))$ with large norm. This ensures the existence of a feature $k \not\in D_{\tau(E)}$ that has not yet been modified up until time $\tau(E)$, and that is assigned a large weight by our least-square solution. In turn, this feature is more likely to be modified in future epochs.

Our main result in this section shows that the tie-breaking rule of Algorithm~\ref{alg: tie-breaking} eventually incentivizes the agents to modify all $d$ features, allowing for accurate recovery of $\beta^*$ in its entirety. The intuition behind our algorithm is to choose a tie-breaking rule that puts enough weight on directions that have not yet been explored, incentivizing agents to explore them.

\begin{theorem}[Recovery Guarantee with Tie-Breaking Scheme (Algorithm~\ref{alg: tie-breaking})]\label{thm: tie-breaking}
Suppose the epoch size satisfies $n \geq \frac{\kappa d^2}{\lambda} \sqrt{2T \log(24d/ \delta)}$, and take $\alpha$ to be
\[
\alpha \geq 
\gamma \left(\sqrt{d} + \frac{K d\sqrt{2T \log(8d/\delta)}}{\lambda n}\right),
\]
where $\gamma,~K,~\kappa,~\lambda$ are instance-specific constants that only depend on $\sigma$, $\C$, $\Sigma$, and $\lambda > 0$. If $T \geq d n$, we have with probability at least $1-\delta$ that at the end of the last epoch $T/n$,
\[
\left\Vert  \hatb_{T/n} - \beta^* \right\Vert_2 
\leq \frac{K \sqrt{2dT \log(8d/\delta)}}{\lambda n},
\]
under the tie-breaking rule of Algorithm~\ref{alg: tie-breaking} .
\end{theorem}
\begin{remark}
The bound in Theorem \ref{thm: tie-breaking} provides guidance for selecting the epoch length, so as to ensure optimal recovery guarantees. Under the natural assumption that $T>>d$, the optimal recovery rate is achieved when roughly $n = \Theta(T/d)$. This results in an  $O(d \sqrt{(d \log d) / T})$ upper bound on the $\ell_2$ distance between the recovered regression parameters and $\beta^*$.
\end{remark}

\begin{proof}[Proof sketch of Theorem \ref{thm: tie-breaking}]
Full proof in Appendix~\ref{app: tie-breaking}. For $\alpha$ arbitrarily large, the norm of $\hat{\beta}$ becomes arbitrarily large. Because at the end of epoch $E$, $\hat{\beta}_{E}$ guarantees accurate recovery of all features modified up until time $E n$, it must be that $\hat{\beta}_E(k)$ is arbitrarily large for some feature $k$ that has not yet been modified. In turn, this feature is modified in epoch $E+1$. After $d$ epochs, and in particular for $T \geq d n$, this leads to $D_T = [d]$. The recovery guarantee of Theorem~\ref{thm: accurate_recovery} then applies to all features.
\end{proof}
\section{Conclusion}

This work takes a first step towards illuminating a phenomenon we believe is both surprising and worthy of further study: strategic agents may in fact help a learner in better understanding the underlying structure of a classification problem. As an immediate implication, the recovery guarantees we have proven provide the learner with knowledge regarding how to choose good incentives, laying the ground for individual improvement, rather than gaming.
In future work, it would be natural to explore this interaction in richer and more complex settings.

\section*{Acknowledgments}
Part of this work was done while the authors were visiting the Simons Institute for the Theory of Computing.
The work of Yahav Bechavod and Katrina Ligett was supported in part by Israel Science Foundation (ISF) grants \#1044/16 and 2861/20, the United States Air Force and DARPA under contracts FA8750-16-C-0022 and FA8750-19-2-0222, and the Federmann Cyber Security Center in conjunction with the Israel national cyber directorate. Yahav Bechavod was also supported in part by the Apple Scholars in AI/ML PhD Fellowship. Katrina Ligett was also funded in part by in part by a grant from Georgetown University and Simons Foundation Collaboration 733792. Zhiwei Steven Wu was supported in part by the NSF FAI Award  \#1939606, a Google Faculty Research Award, a J.P. Morgan Faculty Award, a Facebook Research Award, and a Mozilla Research Grant.  Juba Ziani was supported in part by the Inaugural PIMCO Graduate Fellowship at Caltech, the National Science Foundation through grant CNS-1518941, as well as the Warren Center for Network and Data Sciences at the University of Pennsylvania.
Any opinions, findings and conclusions or recommendations expressed in this material are those of the author(s) and do not necessarily reflect the views of the United States Air Force and DARPA. We thank Mohammad Fereydounian and Aaron Roth for useful discussions.

\bibliographystyle{plain}
\bibliography{references.bib}

\begin{thebibliography}{10}

\bibitem{bruckner2012}
Michael Br{\"u}ckner, Christian Kanzow, and Tobias Scheffer.
\newblock Static prediction games for adversarial learning problems.
\newblock {\em Journal of Machine Learning Research}, 13(Sep):2617--2654, 2012.

\bibitem{CDP15}
Yang Cai, Constantinos Daskalakis, and Christos~H. Papadimitriou.
\newblock Optimum statistical estimation with strategic data sources.
\newblock In {\em COLT}, 2015.

\bibitem{chara2}
Yiling Chen, Yang Liu, and Chara Podimata.
\newblock Grinding the space: Learning to classify against strategic agents.
\newblock {\em arXiv preprint arXiv:1911.04004}, 2019.

\bibitem{chara1}
Yiling Chen, Chara Podimata, Ariel~D Procaccia, and Nisarg Shah.
\newblock Strategyproof linear regression in high dimensions.
\newblock In {\em Proceedings of the 2018 ACM Conference on Economics and
  Computation}, pages 9--26, 2018.

\bibitem{CIL15}
Rachel Cummings, Stratis Ioannidis, and Katrina Ligett.
\newblock Truthful linear regression.
\newblock In {\em COLT}, 2015.

\bibitem{dalvi2004}
Nilesh Dalvi, Pedro Domingos, Sumit Sanghai, and Deepak Verma.
\newblock Adversarial classification.
\newblock In {\em Proceedings of the tenth ACM SIGKDD international conference
  on Knowledge discovery and data mining}, pages 99--108, 2004.

\bibitem{DFP10}
Ofer Dekel, Felix Fischer, and Ariel~D. Procaccia.
\newblock Incentive compatible regression learning.
\newblock {\em Journal of Computer and System Sciences}, 76(8):759 -- 777,
  2010.

\bibitem{Dong17}
Jinshuo Dong, Aaron Roth, Zachary Schutzman, Bo~Waggoner, and Zhiwei~Steven Wu.
\newblock Strategic classification from revealed preferences.
\newblock In {\em Proceedings of the 2018 ACM Conference on Economics and
  Computation}, pages 55--70, 2018.

\bibitem{GLRS14}
Arpita Ghosh, Katrina Ligett, Aaron Roth, and Grant Schoenebeck.
\newblock Buying private data without verification.
\newblock In {\em Proceedings of the Fifteenth ACM Conference on Economics and
  Computation}, EC '14, pages 931--948, 2014.

\bibitem{nika}
Nika Haghtalab, Nicole Immorlica, Brendan Lucier, and Jack Wang.
\newblock Maximizing welfare with incentive-aware evaluation mechanisms.
\newblock Technical report, working paper, 2020.

\bibitem{Hardt15}
Moritz Hardt, Nimrod Megiddo, Christos Papadimitriou, and Mary Wootters.
\newblock Strategic classification.
\newblock In {\em Proceedings of the 2016 ACM conference on innovations in
  theoretical computer science}, pages 111--122, 2016.

\bibitem{HIM14}
Thibaut Horel, Stratis Ioannidis, and S.~Muthukrishnan.
\newblock Budget feasible mechanisms for experimental design.
\newblock In {\em LATIN 2014: Theoretical Informatics}, Lecture Notes in
  Computer Science, pages 719--730, 2014.

\bibitem{lily2019}
Lily Hu, Nicole Immorlica, and Jennifer~Wortman Vaughan.
\newblock The disparate effects of strategic manipulation.
\newblock In {\em Proceedings of the Conference on Fairness, Accountability,
  and Transparency}, pages 259--268, 2019.

\bibitem{IL13}
Stratis Ioannidis and Patrick Loiseau.
\newblock Linear regression as a non-cooperative game.
\newblock In {\em Web and Internet Economics}, pages 277--290, 2013.

\bibitem{Kleinberg18}
Jon Kleinberg and Manish Raghavan.
\newblock How do classifiers induce agents to invest effort strategically?
\newblock In {\em Proceedings of the FAT*}, pages 825--844, 2019.

\bibitem{miller2019}
John Miller, Smitha Milli, and Moritz Hardt.
\newblock Strategic adaptation to classifiers: A causal perspective.
\newblock {\em arXiv preprint arXiv:1910.10362}, 2019.

\bibitem{milli2019}
Smitha Milli, John Miller, Anca~D Dragan, and Moritz Hardt.
\newblock The social cost of strategic classification.
\newblock In {\em Proceedings of the Conference on Fairness, Accountability,
  and Transparency}, pages 230--239, 2019.

\bibitem{moritz_performative}
Juan~C Perdomo, Tijana Zrnic, Celestine Mendler-D{\"u}nner, and Moritz Hardt.
\newblock Performative prediction.
\newblock {\em arXiv preprint arXiv:2002.06673}, 2020.

\bibitem{PP04}
Javier Perote and Juan Perote-Pena.
\newblock Strategy-proof estimators for simple regression.
\newblock In {\em Mathematical Social Sciences 47}, pages 153--176, 2004.

\bibitem{yo}
Yonadav Shavit, Benjamin Edelman, and Brian Axelrod.
\newblock Learning from strategic agents: Accuracy, improvement, and causality.
\newblock {\em arXiv preprint arXiv:2002.10066}, 2020.

\bibitem{planck1}
Behzad Tabibian, Stratis Tsirtsis, Moein Khajehnejad, Adish Singla, Bernhard
  Sch{\"o}lkopf, and Manuel Gomez-Rodriguez.
\newblock Optimal decision making under strategic behavior.
\newblock {\em arXiv preprint arXiv:1905.09239}, 2019.

\bibitem{planck2}
Stratis Tsirtsis and Manuel Gomez-Rodriguez.
\newblock Decisions, counterfactual explanations and strategic behavior.
\newblock {\em arXiv preprint arXiv:2002.04333}, 2020.

\bibitem{berk2019}
Berk Ustun, Alexander Spangher, and Yang Liu.
\newblock Actionable recourse in linear classification.
\newblock In {\em Proceedings of the Conference on Fairness, Accountability,
  and Transparency}, pages 10--19, 2019.

\end{thebibliography}

\newpage
\appendix 

\section{Proof of Theorem~\ref{thm: accurate_recovery}}\label{app: accurate_recovery}

\subsection{Preliminaries}
\subsubsection{Useful concentration}

Our proof will require applying the following concentration inequality, derived from Azuma's inequality: 
\begin{lemma}\label{lem: azuma}
Let $W_1, \ldots,W_\tau$ be random variables in $\reals$ such that $\left\vert W_t \right\vert \leq W_{max}$. Suppose for all $t \in [\tau]$, for all $w_1,\ldots,w_{t-1}$,
\[
\E \left[W_t \middle| W_{t-1} = w_{t-1},\ldots,W_1 = w_1 \right] = 0.
\]
Then, with  at least $1-\delta$,
\[
\left\vert \sum_{t=1}^\tau W_t \right\vert \leq W_{max} \sqrt{2\tau \log(2/\delta)}.
\]
\end{lemma}

\begin{proof}
This is a reformulated version of Azuma's inequality. To see this, define 
\[
Z_t = \sum_{i = 1}^t W_i~~\forall t,
\]
and initialize $Z_0 = 0$. We start by noting that for all $t \in [\tau]$, since
\[
Z_t = \sum_{i = 1}^t W_i = W_t + \sum_{i = 1}^{t-1} W_i = W_t + Z_{t-1},
\]
we have
\begin{align*}
\E \left[Z_t \middle| Z_{t-1}, \ldots, Z_1\right]
&= \E \left[W_t \middle| Z_{t-1}, \ldots, Z_1\right]
+ \E \left[Z_{t-1} \middle| Z_{t-1}, \ldots, Z_1 \right]
\\&= \E \left[W_t \middle| Z_{t-1}, \ldots, Z_1\right]
+ Z_{t-1}.
\end{align*}
Further, it is easy to see that $Z_{i} = z_{i}~\forall i \in [t-1]$ if and only if $W_{i} = z_{i}-z_{i-1}~\forall i \in [t-1]$, hence
\begin{align*}
\E \left[W_t \middle| Z_{t-1} = z_{t-1}, \ldots, Z_1 = z_1\right] 
=\E \left[W_t \middle| W_i = z_i - z_{i-1}~\forall i \in [t-1] \right]
=0.
\end{align*}
Combining the last two equations implies that
\begin{align*}
\E \left[Z_t \middle| Z_{t-1}, \ldots, Z_1\right] = Z_{t-1},
\end{align*}
and the $Z_t$'s define a martingale. Since for all $t$,
\[
\left\vert Z_t - Z_{t-1} \right\vert = \left\vert W_t \right\vert \leq W_{max},
\]
we can apply Azuma's inequality to show that with probability at least $1-\delta$,
\[
\left\vert Z_\tau - Z_0 \right\vert \geq W_{max} \sqrt{2\tau \log(2/\delta)},
\]
which immediately gives the result.
\end{proof}

\subsubsection{Sub-space decomposition and projection}
We will also need to divide $\reals^d$ in several sub-spaces, and project our observations to said subspaces. 

\paragraph{Sub-space decomposition} We focus on the sub-space generated by the non-modified features $x_t$'s and the sub-space generated by the feature modifications $\Delta_t$'s. 
We let $r$ be the rank of $\Sigma$, and let $\lambda_r \geq \ldots \geq \lambda_1 > 0$ be the non-zero eigenvalues of $\Sigma$. Further, we let $f_1,\ldots,f_r$ be the unit eigenvectors (i.e., such that $\Vert f_1 \Vert_1 = \ldots = \Vert f_r \Vert_1 = 1$) corresponding to eigenvalues $\lambda_1,\ldots,\lambda_r$ of $\Sigma$. As $\Sigma$ is a symmetric matrix, $f_1,\ldots,f_r$ are orthonormal. We abuse notations in the proof of Theorem~\ref{thm: accurate_recovery} and denote $\Sigma = \spn(f_1,\ldots,f_r)$ when clear from context.

For all $k$, let $e_k$ be the unit vector such that $e_k(k) = 1$ and $e_k(j) = 0~\forall j \neq k$. At time $\tau$, we denote $\cD_\tau = \spn \left(e_k\right)_{k \in D_\tau}$ the sub-space of $\reals^d$ spanned by the features in $D_\tau$.

Finally, we let \[
\cV_\tau = \Sigma + \cD_\tau = \spn \left(f_1,\ldots,f_{r} \right) + \spn \left(e_k\right)_{k \in D_\tau}
\]
be the Minkowski sum of sub-spaces $\Sigma$ and $\cD_\tau$.

\paragraph{Projection onto sub-spaces} For any vector $z$, sub-space $\mathcal{H}$ of $\reals^d$, we write $z = z(\mathcal{H}) + z (\mathcal{H}^\bot)$ where $z(\mathcal{H})$ is the projection of $z$ onto sub-space $\mathcal{H}$, i.e. is uniquely defined as
\[
z(\mathcal{H}) = \sum_{q \in B} (z^\top q) q
\]
for any orthonormal basis $B$ of $\mathcal{H}$. We also let $z(\mathcal{H}^\bot)$ be the projection on the orthogonal complement $\mathcal{H}^\bot$. In particular, $z(\mathcal{H})$ is orthogonal to $z(\mathcal{H}^\bot)$. Further, we write $\barX_\tau(\mathcal{H})$ the matrix whose rows are given by $\barx_t(\mathcal{H})^\top$ for all $t \in [\tau]$.

\subsection{Main Proof}

\paragraph{Characterization of the least-square estimate via first-order conditions}
First, for any least square solution $\hatb_E$ at time $\tau(E)$, we write the first order conditions solved by $\hatb_E\left(\cV_{\tau(E)}\right)$, the projection of $\hatb_E$ on sub-space $\cV_{\tau(E)}$. We abuse notations to let $\varepsilon_{\tau(E)} \triangleq \left(\varepsilon_t\right)_{t \in [\tau(E)]}$ the vector of all $\varepsilon_t$'s up until time $\tau(E)$, and state the result as follows:
\begin{lemma}[First-order conditions projected onto $\cV_{\tau(E)}$]\label{lem: FOC}
Suppose $\hatb_E \in LSE(\tau(E))$. Then,
\[
\left(\barX_{\tau(E)}\left(\cV_{\tau(E)}\right)^\top \barX_{\tau(E)}\left(\cV_{\tau(E)}\right) \right) \left( \hatb_E\left(\cV_{\tau(E)}\right) - \beta^*\left(\cV_{\tau(E)}\right) \right) =\barX_{\tau(E)} \left(\cV_{\tau(E)}\right)^\top \varepsilon_{\tau(E)}.
\]
\end{lemma}

\begin{proof}
For simplicity of notations, we drop all $\tau(E)$ indices and subscripts in this proof. Remember that
\[
LSE = \argmin_{\beta} \left(\barX \beta - \barY \right)^\top \left(\barX \beta - \barY \right).
\]
Since $\hatb_E \in LSE$, it must satisfy the first order conditions given by
\[
2 \barX^\top \left(\barX \hatb_E - \barY\right) = 0,
\]
which can be rewritten as 
\[
\barX^\top \barX \hatb_E = \barX^\top \barY.
\]
Second, we note that for all $t$, $x_t \in \spn(f_1,\ldots,f_r)$ and $\Delta_t \in \spn\left(\left(e_k\right)_{k \in D }\right)$ (by definition of $D$). This immediately implies, in particular, that $\barx_t = x_t + \Delta_t \in \cV$. In turn, $\barx_t\left(\cV\right) = \barx_t$ for all $t$, and 
\[
\barX = \barX \left(\cV\right).
\]
As such, the first order condition can be written 
\[
\barX\left(\cV\right)^\top \barX\left(\cV\right) \hatb_E = \barX\left(\cV\right)^\top \barY.
\]
Now, we remark that
\begin{align*}
\barX\left(\cV\right)^\top \barX\left(\cV\right) \hatb_E
&= \sum_{t\in S} \barx_t\left(\cV\right) \barx_t\left(\cV\right)^\top \hatb_E
\\&=  \sum_{t\in S} \barx_t\left(\cV\right) \barx_t\left(\cV\right)^\top \hatb_E\left(\cV\right) +  \sum_{t\in S} \barx_t\left(\cV\right) \barx_t\left(\cV\right)^\top \hatb_E(\cV^\bot)
\\&= \sum_{t\in S} \barx_t\left(\cV\right) \barx_t\left(\cV\right)^\top \hatb_E\left(\cV\right) 
\\&= \barX\left(\cV\right)^\top \barX\left(\cV\right) \hatb_E\left(\cV\right),
\end{align*}
where the second-to-last equality follows from the fact that $\cV$ and $\cV^\bot$ are orthogonal, which immediately implies $\barx_t\left(\cV\right)^\top \hatb_E(\cV^\bot) = 0$ for all $t$. To conclude the proof, we note that $\barY = \barX^\top \beta^* + \varepsilon = \barX\left(\cV\right)^\top \beta^*\left(\cV\right) + \varepsilon$. Plugging this in the above equation, we obtain that
\[
\barX\left(\cV\right)^\top \barX\left(\cV\right) \hatb_E\left(\cV\right) 
= \barX\left(\cV\right)^\top \barX\left(\cV\right)^\top \beta^*\left(\cV\right) + \barX\left(\cV\right)^\top \varepsilon.
\]
This can be rewritten
\begin{align*}
\left(\barX\left(\cV\right)^\top \barX\left(\cV\right) \right) \left( \hatb_E\left(\cV\right) - \beta^*\left(\cV\right) \right) =\barX\left(\cV\right)^\top \varepsilon,
\end{align*}
which completes the proof.
\end{proof}

\paragraph{Upper-bounding the right-hand side of the first order conditions} We now use concentration to give an upper bound on a function of the right-hand side of the first order conditions, 
\[
\left( \hatb_E\left(\cV_{\tau(E)}\right) - \beta^*\left(\cV_{\tau(E)}\right) \right)^\top \barX_{\tau(E)} \left(\cV_{\tau(E)}\right)^\top \varepsilon_{\tau(E)}.
\]

\begin{lemma}\label{lem: upper_bound_norm}
With probability at least $1-\delta$,
\begin{align*}
&\left( \hatb_E\left(\cV_{\tau(E)}\right) - \beta^*\left(\cV_{\tau(E)}\right) \right)^\top \barX_{\tau(E)}\left(\cV_{\tau(E)}\right)^\top \varepsilon
\\&\leq \left\Vert  \hatb_E\left(\cV_{\tau(E)}\right) - \beta^*\left(\cV_{\tau(E)}\right) \right\Vert_2 \cdot K' \sqrt{d\tau(E) \log(2d/\delta)}.
\end{align*}
where $K'$ is a constant that only depends on the distribution of costs and the bound $\sigma$ on the noise.
\end{lemma}

\begin{proof}
Pick any $k \in [d]$, and define $W_t = \barx_t(k) \varepsilon_t$. First, we remark that 
\[
\left\vert \barx_t(k) \right\vert 
\leq \left\vert x_t(k) \right\vert + \left\vert \Delta_t(k) \right\vert
\leq 1 + \max_{k \in [d],~i \in [l]} \frac{B^i}{c^i(k)}.
\]
In turn, $\left\vert W_t \right\vert \leq K'$ where 
\[
K' \triangleq \left( 1 + \max_{k \in [d],~i \in [l]} \frac{B^i}{c^i(k)}\right) \sigma.
\] Further, note that both $x_t(k)$ and $\varepsilon_t$ are independent of the history of play up through time $t-1$, hence of $W_1,\ldots,W_{t-1}$, and that $\varepsilon_t$ is further independent of $\Delta_t$ (the distribution of $\Delta_t$ is a function of the currently posted $\hatb_{E-1}$ only, which only depends on the previous time steps). Noting that if $A,B,C$ are random variables, we have 
\begin{align*}
\E_{A,B}\left[AB|C=c\right]  
&= \sum_{a} \sum_{b} ab \Pr\left[A = a, B= b | C=c\right]
\\&=\sum_a \sum_b ab \Pr\left[A = a | B=b,C=c\right] \Pr\left[B=b |C=c\right]
\\&= \sum_b b \left(\sum_a a \Pr\left[A = a | B=b,C=c\right]\right) \Pr\left[B=b |C=c\right]
\\&=\sum_b b \E_A\left[A|B=b,C=c\right] \Pr\left[B=b |C=c\right]
\\&= \E_B \left[\E_A\left[A|B,C=c\right] B | C=c\right],
\end{align*}
and applying this with $A = \varepsilon_t$, $B = \Delta_t(k)$, $C = W_1 \cap \ldots \cap W_{t-1}$, we obtain
\begin{align*}
\E \left[W_t \middle| W_{t-1},\ldots,W_1\right]
&= \E \left[\barx_t(k) \varepsilon_t \middle| W_{t-1},\ldots,W_1 \right] 
\\&= \E \left[x_t(k) \varepsilon_t \middle| W_{t-1},\ldots,W_1 \right] + \E \left[\Delta_t(k) \varepsilon_t \middle| W_{t-1},\ldots,W_1 \right]
\\&= \E \left[x_t(k) \varepsilon_t \right] + \E_{\Delta_t} \left[ \E_{\varepsilon_t} \left[\varepsilon_t \middle| \Delta_t(k),W_{t-1},\ldots,W_1 \right] \cdot \Delta_t(k) \middle| W_{t-1},\ldots,W_1\right]
\\& = \E_{x_t} \left[x_t(k) \cdot \E_{\varepsilon} \left[ \varepsilon_t | x_t(k) \right]\right] + \E_{\Delta_t} \left[\Delta_t(k) \cdot \E_{\varepsilon_t} \left[\varepsilon_t \right] \middle| W_{t-1},\ldots,W_1\right]
\\& = 0,
\end{align*}
since $\E_{\varepsilon_t} \left[\varepsilon_t \right] = 0$ and $\E_{\varepsilon} \left[ \varepsilon_t | x_t(k) \right] = 0$. Hence, we can apply Lemma~\ref{lem: azuma} and a union bound over all $d$ features to show that with probability at least $1-\delta$, 
\[
\sum_{t=1}^{\tau(E)} \barx_t(k) \varepsilon_t \geq - K' \sqrt{2\tau(E) \log(2d/\delta)}~~\forall k \in [d].
\]
By Cauchy-Schwarz, we have
\begin{align*}
\left( \hatb_E\left(\cV\right) - \beta^*\left(\cV\right) \right)^\top \sum_{t=1}^{\tau(E)} \barx_t \varepsilon_t
&\leq \left\Vert  \hatb_E\left(\cV\right) - \beta^*\left(\cV\right) \right\Vert_2 \cdot \left\Vert \sum_{t=1}^{\tau(E)} \barx_t \varepsilon_t \right\Vert_2
\\&\leq \left\Vert  \hatb_E\left(\cV\right) - \beta^*\left(\cV\right) \right\Vert_2 \sqrt{\sum_{k=1}^d \left(\sum_t \barx_t(k) \varepsilon_t\right)^2}
\\&\leq \left\Vert  \hatb_E\left(\cV\right) - \beta^*\left(\cV\right) \right\Vert_2 \cdot K'\sqrt{2 d\tau(E) \log(2d/\delta)}.
\end{align*}
\end{proof}

\paragraph{Strong convexity of the mean-squared error in sub-space $\cV(\tau(E))$}

We give a lower bound on the eigenvalues of $\barX^\top \barX$ on sub-space $\cV(\tau(E))$, so as to show that at time $\tau(E)$, any least square solution $\hatb_E$ satisfies
\begin{align*}
\left( \hatb_E\left(\cV_{\tau(E)}\right) - \beta^*\left(\cV_{\tau(E)}\right) \right)^\top \barX_{\tau(E)}\left(\cV_{\tau(E)}\right)^\top \barX_{\tau(E)}\left(\cV_{\tau(E)}\right) \left( \hatb_E\left(\cV_{\tau(E)}\right) - \beta^*\left(\cV_{\tau(E)}\right) \right) 
\\\geq \Omega(n) \left\Vert \hatb_E\left(\cV_{\tau(E)}\right) - \beta^*\left(\cV_{\tau(E)}\right) \right\Vert_2^2.
\end{align*}

To do so, we will need the following concentration inequalities:
\begin{lemma}\label{lem: concentration}
Suppose $\E\left[x_t\right] = 0$. Fix $\tau(E)= En$ for some $E \in \mathbb{N}$. With probability at least $1-\delta$, we have that
\begin{align*}
\sum_{t=1}^{\tau(E)} z^\top x_t x_t^\top z \geq \left(\lambda_r \tau(E) - 2 r d \sqrt{\tau(E) \log(6r/\delta)} \right) \Vert z \Vert_2^2~~\forall z \in \Sigma,
\end{align*}
and
\begin{align*}
\sum_{t=1}^{\tau(E)} z^\top \Delta_t \Delta_t^\top z 
\geq  \left( \min_{i,k} \left\{ \pi^i \left(\frac{B^i}{c^i(k)}\right)^2 \right\} n - \left(\max_{i,k} \left\{ \frac{B^i}{c^i(k)}\right\}\right)^2 \sqrt{2 n \log(6d/\delta)}\right)\Vert z \Vert_2^2 ~~\forall z \in \cD_{\tau(E)}
\end{align*}
and
\begin{align*}
\sum_{t=1}^{\tau(E)} z^\top x_t \Delta_t^\top z \geq -2 \max_{i,k} \left\{ \frac{B^i}{c^i(k)}\right\} d \sqrt{\tau(E) \log(6d/\delta)} \Vert z \Vert_2^2~~\forall z \in \mathbb{R}^d.
\end{align*}
\end{lemma}

\begin{proof}
Deferred to Appendix~\ref{app: concentration}.
\end{proof}

We will also need the following statement on the norm of the projections of any $z \in \cV$ to $\cD$ and $\Sigma$: 

\begin{lemma}\label{lem: eigenvalue_renorm}
Let
\begin{align*}
    \lambda(\cD,\Sigma) = \inf_{z \in \cD + \Sigma}~& \Vert z(\cD) \Vert_2  + \Vert z(\Sigma) \Vert_2 
    \\\text{s.t.}~~~~&\Vert z \Vert_2 = 1. 
\end{align*}
Then, $\lambda(\cD,\Sigma) > 0$.
\end{lemma}

\begin{proof}
With respect to the Euclidean metric, the objective function is continuous in $z$ (the orthogonal projection operators are linear hence continuous functions of $z$ and $z \to \Vert z \Vert_2$ also is a continuous function), and its feasible set is compact (as it is a sphere in a bounded-dimensional space over real values). By the extreme value theorem, the optimization problem admits an optimal solution, i.e., there exists $z^*$ with $\Vert z^* \Vert_2 = 1$ such that $\lambda(\cD,\Sigma) = \Vert z^*(\cD) \Vert_2  + \Vert z^*(\Sigma) \Vert_2$. Now, supposing $\lambda(\cD,\Sigma) \leq 0$, it must necessarily be the case that $z(\cD) = 0$, $z(\Sigma) = 0$. In particular, this means $z$ is orthogonal to both $\cD$ and $\Sigma$. In turn, $z$ must be orthogonal to every vector in $\cD + \Sigma$; since $z \in \cD + \Sigma$, this is only possible when $z = 0$, contradicting $\Vert z \Vert_2 = 1$.
\end{proof}

We can now move onto the proof of our lower bound for 
\[
\left( \hatb_E\left(\cV_{\tau(E)}\right) - \beta^*\left(\cV_{\tau(E)}\right) \right)^\top \barX_{\tau(E)}\left(\cV_{\tau(E)}\right)^\top \barX_{\tau(E)}\left(\cV_{\tau(E)}\right) \left( \hatb_E\left(\cV_{\tau(E)}\right) - \beta^*\left(\cV_{\tau(E)}\right) \right).
\]
\begin{corollary}\label{cor: lower_bound_norm}
Fix $\tau(E)= En$ for some $E \in \mathbb{N}$. With probability at least $1- \delta$,
\begin{align*}
&\left( \hatb_E\left(\cV_{\tau(E)}\right) - \beta^*\left(\cV_{\tau(E)}\right) \right)^\top \barX_{\tau(E)}\left(\cV_{\tau(E)}\right)^\top \barX_{\tau(E)}\left(\cV_{\tau(E)}\right) \left( \hatb_E\left(\cV_{\tau(E)}\right) - \beta^*\left(\cV_{\tau(E)}\right) \right)
\\&\geq 
\left(\frac{\lambda n}{2} - \kappa' d^2 \sqrt{\tau(E)\log(6d/\delta)}\right) 
\left\Vert \hatb_E\left(\cV_{\tau(E)}\right) - \beta^*\left(\cV_{\tau(E)}\right) \right\Vert_2^2,
\end{align*}
for some constants $\kappa',~\lambda$ that only depend on $\sigma$, $\C$, and $\Sigma$, with $\lambda > 0$.
\end{corollary}

\begin{proof}
Since it is clear from context, we drop all $\tau(E)$ subscripts in the notation of this proof. First, we remark that
\begin{align*}
z^\top \barX^\top \barX z 
&= \sum_t z^\top \barx_t \barx_t^\top z
\\&= \sum_t z^\top x_t x_t^\top z + \sum_t z^\top \Delta_t \Delta_t^\top z + 2 \sum_t z^\top \Delta_t z^\top x_t.
\end{align*}
We have by Lemma~\ref{lem: eigenvalue_renorm} that for all $z \in \cV = \cD + \Sigma$,
\[
\Vert z(\cD)\Vert_2 + \Vert z(\Sigma)\Vert_2
\geq \lambda(\cD,\Sigma) \Vert z \Vert_2.
\]
Let $\lambda(\Sigma) \triangleq \min_{D \subset [d]} \lambda(\cD,\Sigma)$. Since there are finitely many subsets $D$ of $[d]$ (and corresponding sub-spaces $\cD$) and since for all such subsets, $\lambda(\cD,\Sigma) > 0$, we have that $\lambda(\Sigma) > 0$. Further, 
\[
\Vert z(\cD)\Vert_2 + \Vert z(\Sigma)\Vert_2
\geq \lambda(\Sigma) \Vert z \Vert_2.
\]
Therefore, it must be the case that either $\Vert z(\cD)\Vert_2 
\geq \frac{\lambda(\Sigma)}{2} \Vert z \Vert_2$ or $\Vert z(\Sigma)\Vert_2 
\geq \frac{\lambda(\Sigma)}{2} \Vert z \Vert_2$. We divide our proof into the corresponding two cases:

\begin{enumerate}
\item The first case is when $\Vert z(\Sigma) \Vert_2 \geq  \frac{\lambda(\Sigma)}{2} \Vert z \Vert_2$. Then, note that since $z^\top \Delta_t \Delta_t^\top z \geq 0$  always, we have
\begin{align*}
\sum_t z^\top \barx_t \barx_t^\top z
&\geq \sum_t z^\top x_t x_t^\top z + 2 \sum_t z^\top \Delta_t z^\top x_t
\\&=  \sum_t z(\Sigma)^\top x_t x_t^\top z(\Sigma) + 2 \sum_t z^\top \Delta_t z^\top x_t,
\end{align*}
where the last equality follows from the fact that $x_t \in \Sigma$ and $z = z(\Sigma) + z(\Sigma^\bot)$. By Lemma~\ref{lem: concentration}, we get that for some constant $C_1$ that depends only on $\C$, 
\begin{align*}
&\sum_t z^\top \barx_t \barx_t^\top z
\\&\geq \left(\lambda_r \tau(E) - 2 r d \sqrt{\tau(E) \log(6r/\delta)} \right) \Vert z(\Sigma) \Vert_2^2
- C_1 d  \sqrt{\tau(E) \log(6d/\delta)} \Vert z \Vert_2^2 
\\& \geq
\left(\frac{\lambda(\Sigma) \lambda_r}{2} \tau(E) - \lambda(\Sigma) r d \sqrt{\tau(E) \log(6r/\delta)}
- C_1 d \sqrt{\tau(E) \log(6d/\delta)}\right) \Vert z \Vert_2^2
\\&\geq
\left(\frac{\lambda(\Sigma) \lambda_r}{2} \tau(E) - \lambda(\Sigma) d^2 \sqrt{\tau(E) \log(6d/\delta)}
- C_1 d \sqrt{\tau(E) \log(6d/\delta)}\right) \Vert z \Vert_2^2.
\end{align*}
(The second step assumes $\lambda_r \tau(E) - 2 r d \sqrt{\tau(E) \log(6r/\delta)} \geq 0$. When this is negative, the bound trivially holds as $\sum_t z^\top \barx_t \barx_t^\top z \geq 0$.)

\item The second case arises when  $\Vert z(\cD) \Vert_2 \geq  \frac{\lambda(\Sigma)}{2} \Vert z \Vert_2$. Note that
\begin{align*}
\sum_t z^\top \barx_t \barx_t^\top z
&\geq \sum_t z^\top \Delta_t \Delta_t^\top z + 2 \sum_t z^\top \Delta_t z^\top x_t
\\&=  \sum_t z(\cD)^\top \Delta_t \Delta_t^\top z(\cD) + 2 \sum_t z^\top \Delta_t z^\top x_t,
\end{align*}
as $\Delta_t \in \cD$ and $z = z(\cD) + z(\cD^\bot)$. By Lemma~\ref{lem: concentration}, it follows that for some constants $C_2,~C_3$ that only depend on $\C$, 
\begin{align*}
&\sum_t z^\top \barx_t \barx_t^\top z
\\&\geq 
\left(n \min_{i,k} \left\{\pi^i \left(\frac{B^i}{c^i(k)}\right)^2 \right\} - C_2 \sqrt{n \log(6d/\delta)} \right) \left\Vert z(\cD) \right\Vert_2^2
- C_3 d \sqrt{\tau(E) \log(6d/\delta)} \Vert z \Vert_2^2 
\\&\geq 
\left(\frac{\lambda(\Sigma) n}{2} \min_{i,k} \left\{\pi^i \left(\frac{B^i}{c^i(k)}\right)^2\right\} - \frac{\lambda(\Sigma) C_2}{2} \sqrt{n \log(6d/\delta)} - C_3 d \sqrt{\tau(E) \log(6d/\delta)}\right) \Vert z \Vert_2^2
\\&\geq 
\left(\frac{\lambda(\Sigma) n}{2} \min_{i,k} \left\{\pi^i \left(\frac{B^i}{c^i(k)}\right)^2\right\} - \frac{\lambda(\Sigma) C_2}{2} \sqrt{\tau(E) \log(6d/\delta)} - C_3 d \sqrt{\tau(E) \log(6d/\delta)}\right) \Vert z \Vert_2^2.
\end{align*}
\end{enumerate}
Noting that by definition $\lambda_r > 0$ and $\min_{i,k} \left\{\pi^i \left(\frac{B^i}{c^i(k)}\right)^2\right\} > 0$, and picking the worse of the two above bounds on $\sum_t z^\top \barx_t \barx_t^\top z$ concludes the proof with 
\[
\lambda = \frac{\lambda(\Sigma)}{2} \min \left(\lambda_r,\min_{i,k} \left\{\pi^i \left(\frac{B^i}{c^i(k)}\right)^2\right\}\right) > 0.
\]
\end{proof}

We can now prove Theorem~\ref{thm: accurate_recovery}. By Lemma~\ref{lem: FOC}, we have that 
\[
\left(\barX_{\tau(E)}\left(\cV_{\tau(E)}\right)^\top \barX_{\tau(E)}\left(\cV_{\tau(E)}\right) \right) \left( \hatb_E\left(\cV_{\tau(E)}\right) - \beta^*\left(\cV_{\tau(E)}\right) \right) =\barX_{\tau(E)}\left(\cV_{\tau(E)}\right)^\top \varepsilon_{\tau(E)},
\]
which immediately yields
\begin{align*}
&\left( \hatb_E\left(\cV_{\tau(E)}\right) - \beta^*\left(\cV_{\tau(E)}\right) \right)^\top \left(\barX_{\tau(E)}\left(\cV_{\tau(E)}\right)^\top \barX_{\tau(E)}\left(\cV_{\tau(E)}\right) \right) \left( \hatb_E\left(\cV_{\tau(E)}\right) - \beta^*\left(\cV_{\tau(E)}\right) \right) 
\\&= \left( \hatb_E\left(\cV_{\tau(E)}\right) - \beta^*\left(\cV_{\tau(E)}\right) \right)^\top \barX\left(\cV_{\tau(E)}\right)^\top \varepsilon_{\tau(E)}
\end{align*}
by performing matrix multiplication with $\left( \hatb_E\left(\cV_{\tau(E)}\right) - \beta^*\left(\cV_{\tau(E)}\right) \right)^\top$ on both sides on the first-order conditions. Further, by Lemma~\ref{lem: upper_bound_norm}, Corollary~\ref{cor: lower_bound_norm}, and a union bound, we get that with probability at least $1- \delta$, 
\begin{align*}
&\left( \hatb_E\left(\cV_{\tau(E)}\right) - \beta^*\left(\cV_{\tau(E)}\right) \right)^\top \barX_{\tau(E)}\left(\cV_{\tau(E)}\right)^\top \barX_{\tau(E)}\left(\cV_{\tau(E)}\right) \left( \hatb_E\left(\cV_{\tau(E)}\right) - \beta^*\left(\cV_{\tau(E)}\right) \right)
\\&\geq 
\left(\frac{\lambda n}{2} - \kappa' d^2 \sqrt{\tau(E)\log(12d/\delta)}\right) \left\Vert \hatb_E\left(\cV_{\tau(E)}\right) - \beta^*\left(\cV_{\tau(E)}\right) \right\Vert_2^2,
\end{align*}
and 
\begin{align*}
&\left( \hatb_E\left(\cV_{\tau(E)}\right) - \beta^*\left(\cV_{\tau(E)}\right) \right)^\top \barX_{\tau(E)}\left(\cV_{\tau(E)}\right)^\top \varepsilon
\\&\leq \left\Vert  \hatb_E\left(\cV_{\tau(E)}\right) - \beta^*\left(\cV_{\tau(E)}\right) \right\Vert_2 \cdot K' \sqrt{d\tau(E) \log(4d/\delta)}.
\end{align*}
Combining the two above inequalities with the first-order conditions yields
\[
\left\Vert \hatb_E\left(\cV_{\tau(E)}\right) - \beta^*\left(\cV_{\tau(E)}\right) \right\Vert_2 \leq \frac{K' \sqrt{d\tau(E) \log(4d/\delta)}}{\frac{\lambda n}{2} - \kappa' d^2 \sqrt{\tau(E)\log(12d/\delta)}}. 
\]
For 
\[
n \geq \frac{4\kappa' d^2}{\lambda} \sqrt{\tau(E)\log(12d/\delta)},
\]
the bound becomes
\[
\left\Vert \hatb_E\left(\cV_{\tau(E)}\right) - \beta^*\left(\cV_{\tau(E)}\right) \right\Vert_2 \leq \frac{4K' \sqrt{d\tau(E) \log(4d/\delta)}}{\lambda n}.
\]
The proof concludes by letting $K \triangleq 4K'$, $\kappa \triangleq 4\kappa'$  and noting that since $\cD_{\tau(E)} \subset \cV_{\tau(E)}$ by construction, the statement holds true over $\cD_{\tau(E)}$ (projecting onto a subspace cannot increase the $\ell2$-norm).

\subsubsection{Proof of Lemma~\ref{lem: concentration}}\label{app: concentration}
For the first statement, note that for all $k \neq j \leq r$,
\[
\E\left[ f_k ^\top x_t x_t^\top f_j \right] = f_k ^\top \E\left[x_t x_t^\top \right] f_j = \lambda_j f_k^\top f_j ,
\]
as $f_j$ is (by definition) an eigenvector of $\Sigma = \E\left[x_t x_t^\top \right]$ for eigenvalue $\lambda_j$. Note that the $f_j^\top x_t x_t^\top f_k = (f_j^\top x_t)(f_k^\top x_t)$ are random variables that are independent across $t$. Further, by Cauchy-Schwarz, 
\[
\left\vert (f_k^\top x_t)(f_j^\top x_t) \right\vert \leq \Vert f_k \Vert_2 \Vert f_j \Vert_2 \Vert x_t\Vert_2^2 = \Vert x_t\Vert_2^2 \leq d.
\]
Therefore, we can apply Hoeffding with a union bound over the $r^2$ choices of $(f_k,f_j)$ to show that with probability at least $1-\delta'$,
\[
\left\vert \sum_{t=1}^{\tau(E)} f_k^\top x_t x_t^\top f_j - \lambda_j \tau(E) f_k^\top f_j \right\vert \leq d \sqrt{2 \tau(E) \log(2r^2/\delta')}.
\]
Note now that for all $z \in \Sigma$, we can write $z = \sum_{k=1}^r \left(z^\top f_k\right) f_k$, and as such 
\begin{align*}
&\left\vert \sum_{t=1}^{\tau(E)} z^\top x_t x_t^\top z - \sum_{k,j=1}^r (z^\top f_k) (z^\top f_j) \lambda_j \tau(E) f_k^\top f_j \right\vert 
\\&= \left\vert \sum_{t=1}^{\tau(E)} \sum_{k,j=1}^r (z^\top f_k) (z^\top f_j) f_k^\top x_t x_t^\top f_j - \sum_{k,j=1}^r (z^\top f_k) (z^\top f_j) \lambda_j \tau(E) f_k^\top f_j \right\vert 
\\&= \left\vert \sum_{k,j=1}^r (z^\top f_k) (z^\top f_j) \left(\sum_t f_k^\top x_t x_t^\top f_j -  \lambda_j \tau(E) f_k^\top f_j\right) \right\vert
\\&\leq d \sqrt{2 \tau(E) \log(2r^2/\delta')} \sum_{k,j=1}^r |z^\top f_k| |z^\top f_j|
\\&\leq rd \sqrt{2 \tau(E) \log(2r^2/\delta')} \Vert z \Vert_2^2,
\end{align*}
where the last step follows from the fact that by Cauchy-Schwarz, 
\[
\sum_{k=1}^r |z^\top f_k| \leq \sqrt{\sum_{k=1}^r 1^2} \sqrt{\sum_{k=1}^r (z^\top f_k)^2} = \sqrt{r} \Vert z \Vert_2.
\]
Hence, for $z \in \Sigma$, remembering $f_k^\top f_j = 0$ when $k \neq j$ and $f_k^\top f_k = 1$, and noting $\Vert z \Vert_2^2 =  \sum_{k=1}^r (z^\top f_k)^2$, we get that 
\begin{align*}
\sum_{t=1}^{\tau(E)} z^\top x_t x_t^\top z
&\geq \sum_{k,j=1}^r (z^\top f_k) (z^\top f_j) \lambda_j \tau(E) f_k^\top f_j - r d \sqrt{2\tau(E) \log(2r^2/\delta')} \Vert z \Vert_2^2
\\&=\sum_{k=1}^r \lambda_k \tau(E) (z^\top f_k)^2 - r d \sqrt{2\tau(E) \log(2r^2/\delta')} \Vert z \Vert_2^2
\\&\geq \lambda_r \tau(E) \sum_{k=1}^r (z^\top f_k)^2 - r d \sqrt{2\tau(E) \log(2r^2/\delta')} \Vert z \Vert_2^2
\\&= \left(\lambda_r \tau(E) - 2 r d \sqrt{\tau(E) \log(2r/\delta')} \right) \Vert z \Vert_2^2.
\end{align*}

For the second statement, we remind the reader that the costs of modification are such that $\left\vert \Delta_t(k)^2 \right\vert \leq \left(\max_{i,j} \left\{\frac{B^i}{c^i(j)}\right\}\right)^2$, and that within any epoch $\phi$, the $\Delta_t$'s are independent of each other. We can therefore apply Hoeffding's inequality and a union bound (over $k \in D_{\tau(E)} \subset [d]$) to show that with probability at least $1-\delta'$, for any $k \in D_{\tau(E)}$, there exists an epoch $\phi(k) \leq E$ (pick any $\phi$ in which $k$ is modified) such that
\begin{align*}
\sum_{t \in \phi(k)} e_k^\top \Delta_t \Delta_t^\top e_k 
&\geq n \E \left[\Delta_t(k)^2 \right] - \left(\max_{i,j} \left\{\frac{B^i}{c^i(j)}\right\}\right)^2 \sqrt{2n \log(d/\delta')}
\\&\geq n \min_{i \in [l],j \in [d]} \left\{ \pi^i \left(\frac{B^i}{c^i(j)}\right)^2 \right\} - \left(\max_{i,j} \left\{\frac{B^i}{c^i(j)}\right\}\right)^2 \sqrt{2n \log(d/\delta')}.
\end{align*}
The last inequality holds noting that $k$ can be modified in period $\phi(k)$ only if there exists a cost type $i$ on the support of $\C$ such that $k$ is a best response to $\hatb_{\phi(k) - 1}$; in turn, $k$ is modified with probability $\pi^i$ by amount $\Delta(k) = B^i/c^i(k)$, leading to
\[
\E \left[\Delta_t(k)^2 \right] \geq \pi^i \left(\frac{B^i}{c^i(k)}\right)^2. 
\]
Since $\Delta_t(k) \Delta_t(j) = 0$ when $k \neq j$ as a single direction is modified at a time, note that for all $z \in \cD_{\tau(E)}$, we have 
\begin{align*}
&\sum_{t \leq \tau(E)} z^\top \Delta_t \Delta_t^\top z
\\&= \sum_{t \leq \tau(E)} \sum_{k=1}^d \Delta_t(k)^2 z^\top e_k e_k^\top z
\\&= \sum_{k=1}^d \sum_{t \leq \tau(E)} \Delta_t(k)^2 (z^\top e_k)^2
\\&\geq \sum_{k \in D_{\tau(E)}}  \sum_{t \in \phi(k)}  \Delta_t(k)^2 (z^\top e_k)^2
\\&\geq \sum_{k \in D_{\tau(E)}} \left( n \min_{i \in [l],j \in [d]} \left\{ \pi^i \left(\frac{B^i}{c^i(j)}\right)^2 \right\} - \left(\max_{i,j} \left\{\frac{B^i}{c^i(j)}\right\}\right)^2 \sqrt{2n \log(d/\delta')}\right) (z^\top e_k)^2
\\&= \left( n \min_{i \in [l],j \in [d]} \left\{ \pi^i \left(\frac{B^i}{c^i(j)}\right)^2 \right\} - \left(\max_{i,j} \left\{\frac{B^i}{c^i(j)}\right\}\right)^2 \sqrt{2n \log(d/\delta')}\right)  \sum_{k \in D_{\tau(E)}} (z^\top e_k)^2.
\end{align*}
For $z \in \cD_{\tau(E)}$, $\sum_{k \in D_{\tau(E)}} (z^\top e_k)^2 = \Vert z \Vert_2^2$, and the second inequality immediately holds.

Finally, let us prove the last inequality. Take $(k,j) \in [d]^2$, and let us write $W_t = e_k^\top x_t \Delta_t^\top e_j$. First, note that $x_t$ and $\Delta_t$ are independent: in epoch $\phi$, the distribution of $\Delta_t$ is a function of $\hatb_{\phi-1}$ (and $\C$) only, which only depends on the realizations of $x,~\varepsilon,~\Delta$ in previous time steps. Further, $x_t$ is independent of the history of features and modifications up until time $t-1$ included. Hence, it must be the case that
\begin{align*}
\E \left[ W_t \middle|  W_{t-1},\ldots,W_1\right] 
& = \E \left[ \E\left[e_k^\top x_t \middle| \Delta_t, W_{t-1},\ldots, W_1\right]\Delta_t^\top e_j \middle|  W_{t-1},\ldots,W_1 \right]
\\&= \E \left[ \E\left[e_k^\top x_t\right]\Delta_t^\top e_j \middle|  W_{t-1},\ldots,W_1 \right]
\\&= \E \left[ e_k^\top x_t\right] 
\cdot \E \left[\Delta_t^\top e_j \middle| W_{t-1},\ldots,W_1 \right]
\\& = 0,
\end{align*}
where the last equality follows from the fact that $\E\left[x_t \right] = 0$. Further,
\[
\left\vert e_k^\top x_t \Delta_t^\top e_j \right\vert = \vert x_t(k)\vert \vert  \Delta_t(j) \vert \leq \max_{i,k} \left\{\frac{B^i}{c^i(k)}\right\}.
\]
We can therefore apply Lemma~\ref{lem: azuma} and a union bound over all $(k,j) \in [d]^2$ to show that with probability at least $1-\delta'$, 
\[
\left\vert \sum_{t=1}^{\tau(E)} e_k^\top x_t \Delta_t^\top e_j \right\vert \leq \max_{i,k} \left\{\frac{B^i}{c^i(k)}\right\} \sqrt{2\tau(E) \log(2d^2/\delta')}.
\]
In particular, we get that for all $z \in \reals^d$,
\begin{align*}
\left\vert 
\sum_{t \in E} z^\top x_t \Delta_t^\top z 
\right\vert
&= \left\vert 
\sum_{k,j}\sum_{t \in E} (z^\top e_k) (z^\top e_j) e_k^\top x_t \Delta_t^\top e_j 
\right\vert
\\&\leq  \sum_{k,j} |z^\top e_k| |z^\top e_j| \left\vert \sum_{t \in E} e_k^\top x_t \Delta_t^\top e_j\right\vert 
\\&\leq \max_{i,k} \left\{\frac{B^i}{c^i(k)}\right\} \sqrt{2\tau(E) \log(2d^2/\delta')} \left(\sum_{k} |z^\top e_k|\right)^2
\\&\leq 2 d \max_{i,k} \left\{\frac{B^i}{c^i(k)}\right\} \sqrt{\tau(E) \log(2d/\delta')} \Vert z \Vert_2^2,
\end{align*}
where the last step follows from the fact that by Cauchy-Schwarz,
\[
\left(\sum_{k} |z^\top e_k|\right)^2 = \left(\sum_{k} |z(k)| \right)^2 \leq \sum_{k} 1^2 \cdot \sum_{k} z(k)^2 = d \cdot \Vert z \Vert_2^2.
\]

We conclude the proof with a union bound over all three inequalities, taking $\delta' = 3 \delta$.

\section{Proof of Theorem~\ref{thm: tie-breaking}}\label{app: tie-breaking}

We drop the $\tau(E)$ subscripts when clear from context. We first note that $\hat{\beta}_E$ is a least-square solution. 

\begin{claim}\label{clm: is_LSE_solution}
\[
\hat{\beta}_E \in LSE(\tau(E)).
\]
\end{claim}

\begin{proof}
This follows immediately from noting that 
\[
\left(\barX \hat{\beta}_E - \barY \right)^\top \left(\barX \hat{\beta}_E - \barY \right)
= \left(\barX \beta_E - \barY \right)^\top \left(\barX \beta_E - \barY \right),
\]
as $\barX^\top v = \barX(\cU)^\top v = 0$ by definition of $\cU$, and since $v \in \cU^\bot$.
\end{proof}

Second, we show that $\hat{\beta}_E$ has large norm:
\begin{claim}\label{clm: large_norm}
\[
\left\Vert \hat{\beta}_E \right\Vert_2 \geq \alpha.
\]
\end{claim}
\begin{proof}
First, we note that necessarily, $\beta_E \in \cU_{\tau(E)}$. Suppose not, then we can write 
\[
\beta_E = \beta_E\left(\cU_{\tau(E)}\right) + \beta_E\left(\cU_{\tau(E)}^\bot\right),
\]
with $\beta_E\left(\cU_{\tau(E)}^\bot\right) \neq 0$. By the same argument as in Claim~\ref{clm: is_LSE_solution}, $\beta_E\left(\cU_{\tau(E)}\right)$ is a least-square solution. Using orthogonality of $\cU_{\tau(E)}$ and $\cU_{\tau(E)}^\bot$ and the fact that $\left\Vert \beta_E\left(\cU_{\tau(E)}^\bot\right) \right\Vert_2 > 0$, we have
\[
\left\Vert \beta_E \right\Vert^2 =  \left\Vert \beta_E\left(\cU_{\tau(E)}\right) \right\Vert_2^2 + \left\Vert \beta_E\left(\cU_{\tau(E)}^\bot\right) \right\Vert_2^2 > \left\Vert \beta_E\left(\cU_{\tau(E)}\right) \right\Vert_2^2.
\]
This contradicts $\beta_E$ being a minimum norm least-square solution. Hence, it must be the case that $\beta_E \in \cU_{\tau(E)}$. Since $v \in \cU_{\tau(E)}^\bot$, we have that $\beta_E$ and $v$ are orthogonal with $\Vert v \Vert_2 = 1$, implying
\[
\left\Vert \hat{\beta}_E \right\Vert_2^2
= \left\Vert \beta_E \right\Vert_2^2 + \alpha^2 \Vert v \Vert_2^2
\geq \alpha^2.
\]
This concludes the proof.
\end{proof}

We argue that such a solution places a large amount of weight on currently unexplored features:
\begin{lemma} \label{lem: explore}
At time $\tau(E)$, suppose $rank\left(\cU_{\tau(E)}\right) \leq [d]$. Suppose $n \geq \frac{\kappa d^2}{\lambda} \sqrt{\tau(E) \log(12d/\delta')}$. Take any $\alpha$ with 
\[
\alpha \geq 
\gamma \left(\sqrt{d} + \frac{K d\sqrt{T \log(4d/\delta')}}{\lambda n}\right),
\]
where $\gamma$ is a constant that depends only on $\C$. With probability at least $1-\delta'$, there exists $i \in [l]$ and a feature $k \notin D_{\tau(E)}$ with
\[
\frac{\left\vert \hat{\beta}_E(k) \right\vert}{c^i(k)} > \frac{\left\vert \hat{\beta}_E(j) \right\vert}{c^i(j)},~\forall j \in D_{\tau(E)}.
\]
\end{lemma}

\begin{proof}
Since $\hatb_E \in LSE(\tau(E))$, it must be by Theorem~\ref{thm: accurate_recovery} that with probability at least $1-\delta'$,
\begin{align}\label{eq: bound_l2norm}
\begin{split}
\sqrt{\sum_{k \in D}\left( \hatb_E(k) - \beta^*(k) \right)^2}
&\leq \frac{K \sqrt{d\tau(E) \log(4d/\delta')}}{\lambda n}
\\&\leq \frac{K \sqrt{dT \log(4d/\delta')}}{\lambda n}.
\end{split}
\end{align}

First, since $z \to \sqrt{\sum_{k \in D}z(k)^2}$ defines a norm (in fact, the $\ell2$-norm in $\reals^{|D|}$), it must be the case that 
\[
\sqrt{\sum_{k \in D}(z(k)-z'(k))^2} \geq \sqrt{\sum_{k \in D}z(k)^2} - \sqrt{\sum_{k \in D}z'(k)^2}.
\]
In turn, plugging this in Equation~\eqref{eq: bound_l2norm}, we obtain
\begin{align*}
\sqrt{\sum_{k \in D} \hatb_E(k)^2 }
&\leq \sqrt{\sum_{k \in D} \beta^*(k)^2 } +  \frac{K \sqrt{dT \log(4d/\delta')}}{\lambda n}
\\&\leq \Vert \beta^* \Vert_2 +  \frac{K \sqrt{dT \log(4d/\delta')}}{\lambda n}
\\&\leq \sqrt{d} + \frac{K \sqrt{dT \log(4d/\delta')}}{\lambda n}.
\end{align*}
By the triangle inequality and the lemma's assumption, we also have that
\[
\sqrt{\sum_{k \in D} \hatb_E(k)^2} + \sqrt{\sum_{k \notin D}  \hatb_E(k)^2}
\geq \Vert \hatb_E \Vert_2 
\geq \alpha.
\]
Combining the last two equations, we obtain
\[
\sqrt{d} +  \frac{K \sqrt{dT \log(4d/\delta')}}{\lambda n} + \sqrt{\sum_{k \notin D}  \hatb_E(k)^2},
\geq \alpha
\]
which implies that for $\alpha \geq \gamma \left(\sqrt{d} + \frac{K d\sqrt{T \log(4d/\delta')}}{\lambda n}\right)$, we have:
\begin{align*}
\sqrt{\sum_{k \notin D}  \hatb_E(k)^2}
&\geq \alpha - \sqrt{d} - \frac{K \sqrt{dT \log(4d/\delta')}}{\lambda n}
\\&\geq \alpha - \sqrt{d} - \frac{K\sqrt{dT \log(4d/\delta')}}{\lambda n}
\\&\geq \sqrt{d} \left(\gamma-1\right)  \left(1 + \frac{K \sqrt{dT \log(4d/\delta')}}{\lambda n}\right).
\end{align*}

Second, note that Equation~\eqref{eq: bound_l2norm} implies immediately that for any $j \in D_T$, 
\[
\left\vert \hatb_E(j) - \beta^*(j) \right\vert
\leq \frac{K \sqrt{dT \log(4d/\delta')}}{\lambda n},
\]
and in turn,
\begin{align*}
\left\vert \hatb_E(j) \right\vert
\leq \left\vert \beta^*(j) \right\vert + \frac{K \sqrt{dT \log(4d/\delta')}}{\lambda n}
\leq 1 + \frac{K \sqrt{dT \log(4d/\delta')}}{\lambda n}.
\end{align*}
Therefore, 
\begin{align*}
\sqrt{\sum_{k \notin D}  \hatb_E(k)^2}
&\geq \sqrt{d} \left(\gamma - 1\right)\max_{j \in D} \hatb_E(j).  
\end{align*}
Hence, there must exist feature $k \not\in D$ with 
\[
\left \vert \hatb_E(k) \right \vert \geq (\gamma - 1)  \max_{j \in D} \hatb_E(j).  
\]
Picking $\gamma$ such that for some $i \in [l]$,
\[
\gamma - 1 \geq \max_{j \in D} \frac{c^i(k)}{c^i(j)}
\]
yields the result immediately. 
\end{proof}

The proof of Theorem~\ref{thm: tie-breaking} follows directly from Lemma~\ref{lem: explore} and a union bound over the first $d$ epochs. With probability at least $1-d \delta'$, for every epoch $E \in [d]$, there is a feature $k \notin D_{\tau(E)}$ such that for some $i \in [l]$,
\[
\frac{\left\vert \hat{\beta}_E(k) \right\vert}{c^i(k)} > \frac{\left\vert \hat{\beta}_E(j) \right\vert}{c^i(j)}~\forall j \in D_{\tau(E)}.
\]
This implies that there exists $k \in D_{\tau(E+1)}$ but $k \notin D_{\tau(E)}$. Applying this $d$ times, we have that if $T \geq d n$, necessarily $D_T = [d]$. We can then apply Theorem~\ref{thm: accurate_recovery} to then show that with probability at least $1-\delta'$
\begin{align*}
\left\Vert  \hatb_{T/n} - \beta^* \right\Vert_2 
\leq \frac{K \sqrt{d T \log(4d/\delta')}}{\lambda n}.
\end{align*}
Taking a union bound over the two above events and $\delta = 2d \delta'$, we get the theorem statement with probability at least $1 - \delta' \left(d + 1 \right) \geq 1-\delta$.

\end{document}